\documentclass{article}

\usepackage[nonatbib,preprint]{neurips_2021}

\usepackage[square,numbers]{natbib}

\usepackage[utf8]{inputenc} 
\usepackage[T1]{fontenc}    
\usepackage{hyperref}       
\usepackage{url}            
\usepackage{booktabs}       
\usepackage{amsfonts}       
\usepackage{nicefrac}       
\usepackage{microtype}      

\usepackage[jdcorrea]{causality}
\usepackage{causaldiag,jdcorrea}
\usepackage[nameinlink,capitalise]{cleveref}
\usepackage{enumitem}
\usepackage{algorithm,algorithmic}
\usepackage{soul}
\usepackage{bigstrut}

\newcommand{\Val}[2][]{#1\mathfrak{X}_{#2}}

\newcommand{\ctffactor}{ctf-factor}
\newcommand{\Ctffactor}{Ctf-factor}
\newcommand{\ctffactors}{ctf-factors}
\newcommand{\im}[1]{\lVert{#1}\lVert}
\newcommand{\dm}[1]{\widehat{#1}}

\newcommand{\HH}{\mathcal{H}}

\title{Nested Counterfactual Identification\\from Arbitrary Surrogate Experiments}

\author{%
  Juan D.~Correa\\
  Columbia University\\
  \texttt{jdcorrea@cs.columbia.edu} \\
   \And
   Sanghack Lee \\
   Seoul National University \\
   \texttt{sanghack@snu.ac.kr} \\
   \And
   Elias Bareinboim \\
   Columbia University \\
   \texttt{eb@cs.columbia.edu}
}

\begin{document}

\maketitle

\begin{abstract}
The \textit{Ladder of Causation} describes three qualitatively different types of activities an agent may be interested in engaging in, namely, seeing (observational), doing (interventional), and imagining (counterfactual) (Pearl and Mackenzie, 2018). 
The inferential challenge imposed by the causal hierarchy is that data is collected by an agent observing or intervening in a system (layers 1 and 2), while its goal may be to understand what would have happened had it taken a different course of action, contrary to what factually ended up happening (layer 3). 
While there exists a solid understanding of the conditions under which cross-layer inferences are allowed from observations to interventions, the results are somewhat scarcer when targeting counterfactual quantities.  
In this paper, we study the identification of nested counterfactuals from an arbitrary combination of observations and experiments. 
Specifically, building on a more explicit definition of nested counterfactuals, we prove the counterfactual unnesting theorem (CUT), which allows one to map arbitrary nested counterfactuals to unnested ones. 
For instance, applications in mediation and fairness analysis usually evoke notions of direct, indirect, and spurious effects, which naturally require nesting. 
Second, we introduce a sufficient and necessary graphical condition for counterfactual identification from an arbitrary combination of observational and experimental distributions. Lastly, we develop an efficient and complete algorithm for identifying nested counterfactuals; failure of the algorithm returning an expression for a query implies it is not identifiable.
\end{abstract}

\section{Introduction}
Counterfactuals provide the basis for notions pervasive throughout human affairs, such as credit assignment, blame and responsibility, and regret. 
One of the most powerful constructs in human reasoning ---``what if?'' questions---  evokes hypothetical conditions usually contradicting the factual evidence.
Judgment and understanding of critical situations found from medicine to psychology to business involve counterfactual reasoning, e.g.: ``Joe received the treatment and died, would he be alive had he not received it?,'' ``Had the candidate been male instead of female, would the decision from the admissions committee be more favorable?,'' or ``Would the profit this quarter remain within 5\% of its value had we increased the price by 2\%?''. By and large, counterfactuals are key ingredients that go in the construction of explanations about why things happened as they did \cite{pearl:2k, pearl:mackenzie2018}.

The structural interpretation of causality provides proper semantics for representing counterfactuals  \cite[Ch.~7]{pearl:2k}. Specifically, each structural causal model (SCM) $\M$ induces a collection of  distributions related to the activities of seeing (called observational), doing (interventional), and imagining (counterfactual), which together were called the \textit{ladder of causation} \cite{pearl:mackenzie2018,bar:etal2020}.
The ladder is a containment hierarchy; each type of distribution can be put in increasingly refined layers: observational content goes in layer 1; experimental in layer 2; counterfactual in layer 3  (\cref{fig:pch}) .

It is understood that if we have all the information in the world about layer 1, there are still questions about layers 2 and 3 that are unanswerable,  or technically undetermined; further, if we have data from layers 1 and 2, there are still questions in the world about layer 3 that are underdetermined \cite{pearl:mackenzie2018,pearl:2k,bar:etal2020}.

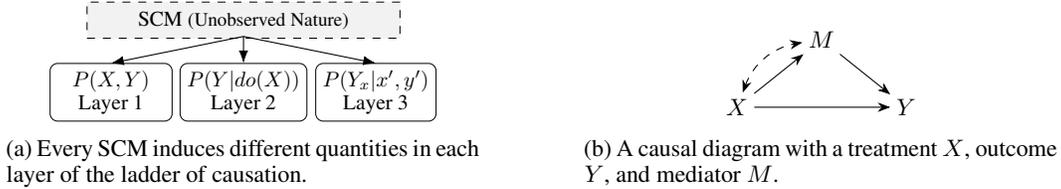
\begin{figure}
    \centering
    \begin{subfigure}[b]{0.45\columnwidth}
        \centering
        \begin{tikzpicture}[scale=0.8, every node/.append style={transform shape}]
            
            \filldraw [fill=gray!10,dashed,line width=0.01mm] (-2.6, 2.5) rectangle (2.6, 1.9);
            
            \node (t1) at (0,2.2) {SCM \small{(Unobserved Nature)}};
            
    	   	\node [draw, line width=0.08mm, align=center,rounded corners=0.1cm] (L1Pv) at (-2.2,1) {$\;\;\, P(X,Y) \;\;\,$\\ Layer 1};
    	   	\node [draw, line width=0.08mm, align=center,rounded corners=0.1cm] (L2Pv) at (0,1) {$P(Y | \doo{X})$\\ Layer 2};
    	   	\node [draw, line width=0.08mm, align=center,rounded corners=0.1cm] (L3Pv) at (2.2,1) {$P(Y_x | x',y')$\\ Layer 3};
    
    	   	\path [-Latex] (t1.south) edge (L1Pv.north);
    	   	\path [-Latex] (t1.south) edge (L2Pv.north);
    	   	\path [-Latex] (t1.south) edge (L3Pv.north);
        \end{tikzpicture}
        \caption{Every SCM induces different quantities in each layer of the ladder of causation.}
        \label{fig:pch}
    \end{subfigure}
    \hfill
	\begin{subfigure}[b]{0.45\columnwidth}
		\centering
		\begin{tikzpicture}[SCM,scale=0.9]
    	   	\node (M) at (1.25,1) {$M$};
    	   	\node (X) at (0,0) {$X$};
    	   	\node (Y) at (2.5,0) {$Y$};

    	    \path [->] (X) edge (M);
    	    \path [->] (M) edge (Y);
    	    \path [->] (X) edge (Y);
    	    \path [<->] (X) edge [dashed, bend left=30] (M);
		\end{tikzpicture}
		\caption{A causal diagram with a treatment $X$, outcome $Y$, and mediator $M$.}
		\label{fig:nde1}
	\end{subfigure}
	\caption{Representation of the ladder of causation and an example of a causal diagram.}
\end{figure}

The inferential challenge in these settings arises since the generating model $\M$ is not fully observed, nor data from all of the layers are necessarily available, perhaps due to the cost or the infeasibility of performing certain interventions. 
One common task found in the literature is to determine the effect of an intervention of a variable $X$ on an outcome $Y$, say $P(Y | do(X))$ (layer 2), using data from observations $P(\*V)$ (layer 1), where $\*V$ is the set of observed variables, and possibly other interventions, e.g., $P(\*V | do(Z))$. Also, qualitative assumptions about the system are usually articulated in the form of a causal diagram $\G$.
This setting has been studied in the literature under the rubric of non-parametric identification from a combination of observations and experiments. Multiple solutions exist, including Pearl's celebrated do-calculus \cite{pearl:95a}, and other increasingly refined solutions that are computationally efficient, sufficient, and necessary  \cite{Spirtes2001,galles:pea95,pearl:rob95,tian:pea02-general-id,shpitser:pea06a,huang:val06-identifiability,bareinboim:pea12-zid,lee:etal19}.

There is growing literature on identification with cross-layer inferences from data in layers 1 and 2 to quantities in layer 3.
For example, a data scientist may be interested in evaluating the effect of a treatment on the group of subjects that receive it instead of those randomly assigned to treatment. This measure is known as the \textit{effect of treatment on the treated} \cite{heckman:92,pearl:2k} and there exists a graphical condition for mapping it to a (layer 2) causal effect \cite{shpitser:pea09}. 
Further, there are also results on the identification of \textit{path-specific effects}, which correspond to counterfactuals that isolate specific paths in the graph \cite{pearl:01}. In particular, \cite{shpitser:she18} provides a complete algorithm for identification from observational data, and \cite{zhang:bar18a} gives identification conditions from observational and experimental data in specific canonical models. 
Moreover, \cite{shpitser:pea07} studied the identification of arbitrary (non-nested) counterfactuals under the assumption that data from experiments in every variable is available. Yet, the problem of identifying such quantities from a subset of the space of all experiments remains open.

For concreteness, consider a counterfactual called \textit{direct effect} in the context of the causal diagram in \cref{fig:nde1}. This quantity quantifies the sensitivity of a variable $Y$ to changes in another variable $X$ while all other factors in the analysis remain fixed.
Suppose $X$ is level of exercise, $M$ cholesterol levels, and $Y$ cardiovascular disease. Exercising can improve cholesterol levels, which in turn affect the chances of developing cardiovascular disease. An interesting question is how much exercise prevents the disease by means other than regulating cholesterol. In counterfactual notation, this is to compare $Y_{x, M_{x}}$ and $Y_{x', M_{x}}$ where $x$ and $x'$ are different values. The first quantity represents the value of $Y$ when $X{=}x$ and $M$ varies accordingly. The second expression is the value $Y$ attains if $X$ is held constant at $x'$ while $M$ still follows $X{=}x$. The difference $E[Y_{x', M_{x}} {-} Y_{x, M_{x}}]$---known as the natural direct effect (NDE)---is non-zero if there is some direct effect of $X$ on $Y$. In this instance, this nested counterfactual is identifiable only if observational data and experiments on $X$ are available.

After all, there is no general identification method for this particular counterfactual family (which also includes indirect and spurious effects) and, more broadly, other arbitrary nested counterfactuals that are well-defined in layer 3.  Our goal is to understand the non-parametric identification of arbitrary nested and conditional counterfactuals when the input consists of any combination of observational and interventional distributions, whatever is available for the data scientist.
More specifically, our contributions are as follows.
\begin{enumerate}[nolistsep,nosep,leftmargin=1.3em]
    \item We look at nested counterfactuals from an SCM perspective and introduce machinery that supports counterfactual reasoning. In particular, we prove the counterfactual unnesting theorem (CUT), which allows one to map any nested counterfactual to an unnested one (\cref{sec:basics}). 
    \item Building on this new machinery, we derive sufficient and necessary graphical conditions and an algorithm to determine the identifiability of marginal nested counterfactuals from an arbitrary combination of observational and experimental distributions (\cref{sec:identification}).
    \item We give a reduction from conditional counterfactuals to marginal ones, and use it to derive a complete algorithm for their identification (\cref{sec:conditional}). 
\end{enumerate}
See the supplemental material for full proofs of the results in the paper.

\subsection{Preliminaries}
We denote variables by capital letters, $X$, and values by small letters, $x$. 
Bold letters, $\*X$ represent sets of variables and $\*x$ sets of values.
The domain of a variable $X$ is denoted by $\Val{X}$.
Two values $\*x$ and $\*z$ are said to be consistent if they share the common values for $\*X\cap\*Z$. We also denote by $\*x\setminus\*Z$ the value of $\*X\setminus\*Z$ consistent with $\*x$ and by $\*x \cap \*Z$ the subset of $\*x$ corresponding to variables in $\*Z$.
We assume the domain of every variable is finite.

Our analysis relies on causal graphs, which we often assign a calligraphic letter, e.g., $\G$, $\HH$, etc.
We denote by $\*V(\HH)$ the set of vertices (i.e., variables) in a graph $\HH$.
Given a graph $\G$, $\G_{\overline{\*W}\underline{\*X}}$ is the result of removing edges coming into variables in $\*W$ and going out from variables in $\*X$. $\G[\*W]$ denotes a vertex-induced subgraph, which includes $\*W$ and the edges among its elements.
We use kinship notation for graphical relationships such as parents, children, descendants, and ancestors of a set of variables.
For example, the set of parents of $\*X$ in $\G$ is denoted by $\Pa[\G]{\*X} := \*X \cup \bigcup_{X\in\*X}\Pa[\G]{X}$. Similarly, we define $\Ch{}$, $\De{}$, and $\An{}$.

To articulate and formalize counterfactual questions, we require a framework that allows us to reason about events corresponding to different \textit{alternative worlds} simultaneously. Accordingly, we employ the Structural Causal Model (SCM) paradigm \cite[Ch.~7]{pearl:2k}. %
An SCM $\M$ is a 4-tuple $\langle \*U, \*V, \mathcal{F}, P(\*u)\rangle$, where $\*U$ is a set of exogenous (latent) variables; $\*V$ is a set of endogenous (observable) variables; $\mathcal{F}$ is a collection of functions such that each variable $V_i \in \*V$ is determined by a function $f_i \in \mathcal{F}$. Each $f_{i}$ is a mapping from a set of exogenous variables $\Ui{i} \subseteq \*U$ and a set of endogenous variables $\Pai{i} \subseteq \*V \setminus \{V_{i}\}$ to the domain of $V_{i}$. The uncertainty is encoded through a probability distribution over the exogenous variables, $P(\*U)$. An SCM $\M$ induces a \emph{causal diagram} $\G$ where every $\*V$ is the set of vertices, there is a directed edge $(V_j \to V_i)$ for every $V_i \in \*V$ and $V_j \in \Pai{i}$, and a bidirected edge $(V_i \dashleftarrow\dashrightarrow V_j)$ for every pair $V_i, V_j \in \*V$ such that $U_i \cap U_j \neq \emptyset$ ($V_i$ and $V_j$ have a common exogenous parent). 

We assume that the underlying model is recursive. That is, there are no cyclic dependencies among the variables. Equivalently, that is to say, that the corresponding causal diagram is acyclic.

The set $\*V$ can be partitioned into subsets called \textit{c-components} \cite{tian:pea02testable-implications} according to a diagram $\G$ such that two variables belong to the same c-component if they are connected in $\G$ by a path made entirely of bidirected edges.

\section{SCMs and Nested Counterfactuals}\label{sec:basics}

Intervening on a system represented by an SCM $\M$ results in a new model differing from $\M$ only on the mechanisms associated with the intervened variables \citep{pearl:94a,dawid:02,dawid:15}.
If the intervention consists on fixing the value of a variable $X$ to a constant $x \in \Val{X}$, it induces a \textit{submodel}, denoted as $\M_{x}$ \cite[Def. 7.1.2]{pearl:2k}.
To formally study nested counterfactuals, we extend this notion to account for models derived from interventions that replace functions from the original SCM with other, not necessarily constant, functions.
\begin{definition}[Derived Model]
\label{def:derived-submodel}
Let $\M$ be an SCM, $\dm{\*U} \subseteq \*U$, $X \in \*V$, and $\dm{X}: \dm{\*U} \rightarrow \Val{X}$ a function. Then, $\M_{\dm{X}}$, called the \emph{derived model} of $\M$ according to $\dm{X}$, is identical to $\M$, except that the function $f_X$ is replaced with a function $\dm{f}_X$ identical to $\dm{X}$.
\end{definition}
This definition is easily extendable to models derived from an intervention on a set $\*X$ instead of a singleton. When $\dm{\*X}$ is a collection of functions $\{\dm{X}: \dm{{\*U}}_X \rightarrow \Val{X}\}_{X \in \*X}$, the derived model $\M_{\dm{\*X}}$ is obtained by replacing each $f_X$ with $\dm{X}$ for $X \in \*X$. Next, we discuss the concept of \textit{potential response} \cite[Def. 7.4.1]{pearl:2k} with respect to derived models.

\begin{definition}[Potential Response]\label{def:potential-response-revisited}
Let $\*X, \*Y \subseteq \*V$ be subsets of observable variables, let $\*u$ be a unit, and let $\dm{\*X}(\*u)$ be a set of functions from $\dm{\*U}_X \to \Val{X}$, for $X \in \*X$ where $\dm{\*U}_X \subseteq \*U$. Then, ${\*Y}_{\*X=\dm{\*X}}(\*u)$ (or ${\*Y}_{\dm{\*X}}(\*u)$, for short) is called the \emph{potential response} of $\*Y$ to $\*X=\dm{\*X}$, and is defined as the solution of ${\*Y}$, for a particular $\*u$, in the derived model $\M_{\dm{\*X}}$.
\end{definition}

A potential response ${Y}_{\dm{\*X}}(\*u)$ describes the value that variable $Y$ would attain for a unit (or individual) $\*u$ if the intervention $\dm{\*X}$ is performed. This concept is tightly related to that of \textit{potential outcome}, but the former explicitly allows for interventions that do not necessarily fix the variables in $\*X$ to a constant value. Averaging over the space of $\*U$, a potential response ${Y}_{\dm{\*X}}(\*u)$ induces a random variable that we will denote simply as ${Y}_{\dm{\*X}}$. If the intervention replaces a function $f_X$ with a potential response of $X$ in $\M$, we say the intervention is \emph{natural}.

When variables are enumerated as $W_1, W_2, \ldots$, we may add square brackets around the part of the subscript denoting interventions. We use $\*W_*$ to denote sets of arbitrary counterfactual variables. Let $\*W_*=\{W_{1[\dm{\*T}_1]}, W_{2[\dm{\*T}_2]}, \ldots\}$ represent a set of counterfactual variables such that $W_i \in \*V$ and $\*T_i \subseteq \*V$ for $i=1,\ldots,l$.
Define $\*V(\*W_*) = \{W \in \*V \mid W_{\dm{\*T}} \in \*W_*\}$, that is, the set of observables that appear in $\*W_{*}$. Let $\*w_*$ represent a vector of values, one for each variable in $\*W_*$ and define $\*w_*(\*X_*)$ as the subset of $\*w_*$ corresponding to $\*X_*$ for any $\*X_* \subseteq \*W_*$.

The probability of any counterfactual event is given by
\begin{align}\label{eq:prob-ctf}
    P(\*Y_*=\*y_*)=\sum_{\left\{\*u \mid \*Y_*(\*u)=\*y_* \right\}}P(\*u),
\end{align}
where the predicate $\*Y_*(\*u)=\*y_*$ means $\bigwedge_{\left\{Y_{\dm{\*X}} \in \*Y_*\right\}}Y_{\dm{\*X}}(\*u)=y$.

When all of the variables in the expression have the same subscript, that is, they belong to the same submodel; we will often denote it as $P_{\*x}(W_1, W_2,\ldots)$.

For most real-world scenarios, having access to a fully specified SCM of the underlying system is unfeasible. Nevertheless, our analysis does not rely on such privileged access but the aspects of the model captured by the causal graph and data samples generated by the unobserved model. 

\subsection{Nested Counterfactuals}

Potential responses can be compounded based on natural interventions. For instance, the counterfactual $Y_{Z_x}(\*u)$ ($Y_{Z=Z_x}(\*u)$) can be seen as the potential response of $Y$ to an intervention that makes $\dm{Z}$ equal to $Z_x$. Notice that $Z_x(\*u)$ is in itself a potential response, but from a different (nested) model. Hence we call $Y_{Z_x}$ a \emph{nested counterfactual}.

Recall the causal diagram in \cref{fig:nde1} and consider once again the NDE as 
\begin{align}
    \textit{NDE}_{x \to x',Z}(Y) = E[Y_{x'Z_x}]-E[Y_x].
\end{align}
The second term is also equal to $Y_{xZ_x}$ as $Z_x$ is consistent with $X=x$, so it is the value $Y$ listens to in $\M_x$. Meanwhile, the first one is indeed related to $P(Y_{x'Z_x})$, the probability of a nested counterfactual.

The following result shows how nested counterfactuals can be written in terms of non-nested ones.
\begin{restatable}[Counterfactual Unnesting Theorem (CUT)]{theorem}{ctfunnesting}
    \label{thm:counterfactualunnesting}
Let $\dm{\*X}, \dm{\*Z}$ be any natural interventions on disjoint sets $\*X, \*Z \subseteq \*V$. Then for $\*Y \subseteq \*V$ disjoint from $\*X$ and $\*Z$, we have
\begin{align}
    P(\*Y_{\dm{\*Z},\dm{\*X}} = \*y) = \sum\nolimits_{\*x \in \Val{\*X}} P(\*Y_{\dm{\*Z},\*x} = \*y, \dm{\*X}=\*x).
\end{align}
\end{restatable}
\begin{proof}[Proof outline]

Based on \cref{eq:prob-ctf}, $P(\*Y_{\dm{\*Z},\dm{\*X}} = \*y)$ can be seen as a sum of the probabilities $P(\*u)$ for the $\*u$ that induce the event $(\*Y_{\dm{\*Z},\dm{\*X}} = \*y)$. Such set of $\*u$ can be partitioned based on the values $\*x=\dm{\*X}(\*u), \*x \in \Val{\*X}$ they induce, which are the same that induce the event $\*Y_{\dm{\*Z},\*x} = \*y, \dm{\*X}=\*x$. Then, the sum over $P(\*u)$ for each subset is equal to the value of the original nested counterfactual.
\end{proof}
For instance, for the model in \cref{fig:nde1} we can write
\begin{align}
    P(Y_{x'Z_x}=y) = \sum\nolimits_{z}P(Y_{x'z}=y,Z_x=z).
\end{align}

As \cref{thm:counterfactualunnesting} allows us to re-write any nested counterfactual in terms of non-nested counterfactuals, we focus on the latter and assume that any given counterfactual is already unnested.

\subsection{Tools for Counterfactual Reasoning}
Before characterizing the identification of counterfactuals from observational and experimental data, we develop from first principles a canonical representation of any such query. First, we extend the notion of ancestors for counterfactual variables, which subsumes the usual one described before.

\begin{definition}[Ancestors, of a counterfactual]\label{def:ctf-ancestors}
Let $Y_{\*x}$ be such that $Y \in \*V, \*X \subseteq \*V$. Then, the set of (counterfactual) ancestors of $Y_\*x$, denoted $\An{Y_\*x}$, consist of each $W_\*z$, such that $W \in \An[\G_{\underline{\*X}}]{Y}$ (which includes $Y$ itself), and $\*z = \*x \cap \An[\G_{\overline{\*X}}]{W}$. 
\end{definition}
For a set of variables $\*W_*$, we define $\An{\*W_*}$ as the union of the ancestors of each variable in the set. That is, $\An{\*W_*}=\bigcup_{W_{\*t} \in \*W_*}\An{W_{\*t}}$.
For instance, in \cref{fig:backdoor}, $\An{Y_x}=\{Y_{x}, Z\}$, $\An{X_{yz}}=\{X_z\}$ and $\An{Y_z}=\{Y_z, X_z\}$ (depicted in \cref{fig:backdoor-ancz}). In \cref{fig:napkin} $\An{Z, Y_z}=\{Y_z, X_z, Z, W\}$ and $\An{Y_x}=\{Y_x\}$ (represented in \cref{fig:napkin-ancx}). 

\begin{figure}
    \centering
	\begin{subfigure}[t]{0.24\columnwidth}
		\centering
		\begin{tikzpicture}[SCM,scale=0.9]
    	   	\node (Z) at (1.25,1) {$Z$};
    	   	\node (X) at (0,0) {$X$};
    	   	\node (Y) at (2.5,0) {$Y$};
    	
    	    \path [->] (X) edge (Y);
    	    \path [->] (Z) edge (X);
    	    \path [->] (Z) edge (Y);
		\end{tikzpicture}
		\caption{``Backdoor'' graph.}
		\label{fig:backdoor}
	\end{subfigure}
	\hfill
	\begin{subfigure}[t]{0.24\columnwidth}
		\centering
		\begin{tikzpicture}[SCM,scale=0.9]
    	   	\node [removed] (Z) at (1.25,1) {$Z$};
    	   	\node (X) at (0,0) {$X_z$};
    	   	\node (Y) at (2.5,0) {$Y_z$};
    	
    	    \path [->] (X) edge (Y);
    	    \path [->,removed] (Z) edge (X);
    	    \path [->,removed] (Z) edge (Y);
		\end{tikzpicture}
		\caption{Graphical representation of the ancestors of $Y_z$.}
		\label{fig:backdoor-ancz}
	\end{subfigure}
	\hfill
	\begin{subfigure}[t]{0.24\columnwidth}
		\centering
		\begin{tikzpicture}[SCM,scale=0.9]
    	   	\node (X) at (0,0) {$X$};
    	   	\node (Z) at (0.8,0.6) {$Z$};
    	   	\node (W) at (1.6,1.2) {$W$};
    	   	\node (Y) at (3.0,0) {$Y$};

    	    \path [->] (W) edge (Z);
    	    \path [->] (Z) edge (X);
    	    \path [->] (X) edge (Y);
    	    \path [<->] (X) edge [dashed, bend left=40] (W);
    	    \path [<->] (W) edge [dashed, bend left=40] (Y);
		\end{tikzpicture}
		\caption{``Napkin'' graph.}
		\label{fig:napkin}
	\end{subfigure}
	\hfill
	\begin{subfigure}[t]{0.24\columnwidth}
		\centering
		\begin{tikzpicture}[SCM,scale=0.9]
    	   	\node [removed] (X) at (0,0) {$X$};
    	   	\node [removed] (Z) at (0.8,0.6) {$Z$};
    	   	\node [removed] (W) at (1.6,1.2) {$W$};
    	   	\node (Y) at (3.0,0) {$Y_x$};

    	    \path [->,removed] (W) edge (Z);
    	    \path [->,removed] (Z) edge (X);
    	    \path [->,removed] (X) edge (Y);
    	    \path [<->] (X) edge [dashed, bend left=40] (W);
    	    \path [<->] (W) edge [dashed, bend left=40] (Y);
		\end{tikzpicture}
		\caption{Graphical representation of the ancestors of $Y_x$.}
		\label{fig:napkin-ancx}
	\end{subfigure}
	\caption{Two causal diagrams and the subgraphs considered when finding sets of ancestors for a counterfactual variable.}
\end{figure}
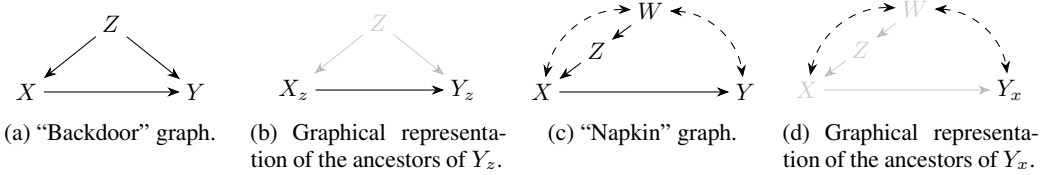

Given a counterfactual variable $\*Y_{\*x}$, it could be the case that some values in $\*x$ become causally irrelevant to $Y$ after the rest of $\*x$ has been fixed. Formally,
\begin{restatable}{lemma}{lemintminimal}\label{lem:int-minimal}
Let $\im{Y_\*x}:=Y_\*z$ where $\*Z = \*X \cap \An[\G_{\overline{\*X}}]{Y}$ and $\*z$ is consistent with $\*x$. Then, $Y_\*x=\im{Y_\*x}$.
\end{restatable}
Moreover, such simplification may reveal counterfactual expressions with equivalent or contradicting events. In \cref{fig:napkin}, $(Y_{xz}=y, Y_{xz'}=y')=(Y_{x}=y, Y_{x}=y')$ which has probability $0$ if $y \neq y'$, or $(Y_{xz}=y, Y_{xz'}=y)$ that is simply $(Y_{x}=y)$. Similarly, the probabilities of counterfactuals events of the form $P(X_x=x')$, $x \neq x'$, and $P(X_x=x)$ are trivially $0$ and $1$ respectively.

For a set of counterfactual variables $\*Y_*$ let $\im{\*Y_*}=\bigcup_{Y_{\*x} \in \*Y_*}\im{Y_{\*x}}$. Notice that each variable in the ancestral set is ``interventionally minimal'' in the sense of \cref{lem:int-minimal}.

Probabilistic and causal inference with graphical models exploits local structure among variables, specifically parent-child relationships, to infer and even estimate probabilities. In particular, Tian \cite{tian:pea02testable-implications} introduced \textit{c-factors} which have proven instrumental in solving many problems in causal inference. We naturally generalize this notion to the counterfactual setting with the following definition.

\begin{definition}[Counterfactual Factor (\ctffactor{})]
A counterfactual factor is a distribution of the form
\begin{align}
    P(W_{1[\pai{1}]}=w_1, W_{2[\pai{2}]}=w_2, \ldots, W_{l[\pai{l}]}=w_l),
\end{align}
where each $W_i \in \*V$ and there could be $W_i = W_j$ for some $i,j \in \{1,\ldots, l\}$.
\end{definition}

For example, for \cref{fig:napkin} $P(Y_x=y, Y_{x'}=y')$, $P(Y_x=y, X_z=x)$ are \ctffactors{} but $P(Y_z=y,Z_w=z)$ is not. Using the notion of ancestrality introduced in \cref{def:ctf-ancestors}, we can factorize counterfactual probabilities as \ctffactors{}.

\begin{restatable}[Ancestral set factorization]{theorem}{thmancctfcomp}\label{thm:ancestral-ctf-comp}
Let $\*W_*$ be an ancestral set, that is, $\An{\*W_*}=\*W_*$, and let $\*w_*$ be a vector with a value for each variable in $\*W_*$. Then,
\begin{align}
    P(\*W_*=\*w_*) = P\left(\bigwedge\nolimits_{W_{\*t} \in \*W_*}W_{\pai{w}}=w\right), \label{eq:ancset-ccfactor}
\end{align}
where each $w$ is taken from $\*w_*$ and $\pai{w}$ is determined for each $W_{\*t} \in \*W_{*}$ as follows:
\begin{enumerate}[label={(\roman*)},topsep=0pt,nolistsep,nosep,leftmargin=2em,itemsep=0.2em]
    \item the values for variables in $\Pai{w} \cap \*T$ are the same as in $\*t$, and
    \item the values for variables in $\Pai{w} \setminus \*T$ are taken from $\*w_*$ corresponding to the parents of $\*W_{\*t}$.
\end{enumerate}
\end{restatable}
\begin{proof}[Proof outline.]
Following a reverse topological order in $\G$, look at each $W_{i\*t_i} \in \*W_*$. Since any parent of $W_i$ not in $\*T_i$ must appear in $\*W_*$, the composition axiom \cite[7.3.1]{pearl:2k} licenses adding them to the subscript. Then, by exclusion restrictions \cite{pearl:95a}, any intervention not involving $\Pa{W_i}$ can be removed to obtain the form in \cref{eq:ancset-ccfactor}.
\end{proof}

For example, consider the diagram in \cref{fig:napkin} and the counterfactual $P(Y_x=y \mid X=x')$ known as the \textit{effect of the treatment on the treated} (ETT) \cite{heckman:92,pearl:2k}. First note that $P(Y_x=y \mid X=x')=P(Y_x=y, X=x')/P(X=x')$ and that $\An{Y_x, X}=\{Y_x, X, Z, W\}$, then
\begin{align}
    P(Y_x=y, X=x')=\sum\nolimits_{z,w}P(Y_x=y, X=x', Z=z, W=w).
\end{align}
Then, by \cref{thm:ancestral-ctf-comp} we can write
\begin{equation}\label{eq:ett-ancestral-set}
    P(Y_x=y, X=x')=\sum\nolimits_{z,w}P(Y_x=y, X_z=x', Z_w=z, W=w).
\end{equation}%
Moreover, the following result describes a factorization of \ctffactors{} based on the c-component structure of the graph, which will prove instrumental in the next section.
\begin{restatable}[Counterfactual factorization]{theorem}{ccompfactorization}\label{thm:cc-comp-factorization}
Let $P(\*W_*=\*w_*)$ be a \ctffactor{}, let $W_1<W_2<\cdots$ be a topological order over the variables in $\G[\*V(\*W_*)]$, and let $\*C_1, \ldots, \*C_k$ be the c-components of the same graph. Define $\*C_{j*}=\{W_{\pai{w}} \in \*W_* \mid W \in \*C_j\}$ and $\*c_{j*}$ as the values in $\*w_*$ corresponding to $\*C_{j*}$, then $P(\*W_*=\*w_*)$ decomposes as
\begin{align}
    P(\*W_*=\*w) = \prod\nolimits_{j}P(\*C_{j*}=\*c_{j*}). \label{eq:\ctffactor{}-factorization}
\end{align}
Furthermore, each factor can be computed from $P(\*W_*=\*w)$ as 
\begin{align}
    P(\*C_{j*}=\*c_{j*}) = \prod_{\{W_i \in \*C_j\}}\frac{
        \sum_{\{w \mid W_{\pai{w}} \in \*W_*, W_i<W\}}P(\*W_*=\*w_*)
    }{
        \sum_{\{w \mid W_{\pai{w}} \in \*W_*, W_{i-1}<W\}}P(\*W_*=\*w_*)
    }.\label{eq:cc-comp-comp}
\end{align}%
\end{restatable}
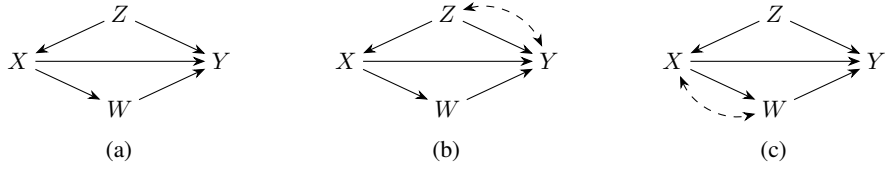
\begin{figure}
    \centering
    \hfill
	\begin{subfigure}[b]{0.24\columnwidth}
		\centering
		\begin{tikzpicture}[SCM,scale=0.9]
    	   	\node (X) at (0,0) {$X$};
    	   	\node (Z) at (1.5,0.7) {$Z$};
    	   	\node (W) at (1.5,-0.7) {$W$};
    	   	\node (Y) at (3.0,0) {$Y$};

    	    \path [->] (Z) edge (X);
    	    \path [->] (Z) edge (Y);
    	    \path [->] (X) edge (Y);
    	    \path [->] (X) edge (W);
    	    \path [->] (W) edge (Y);
		\end{tikzpicture}
		\caption{}
		\label{fig:ex-ctf-measures:a}
	\end{subfigure}
	\hfill
	\begin{subfigure}[b]{0.24\columnwidth}
		\centering
		\begin{tikzpicture}[SCM,scale=0.9]
    	   	\node (X) at (0,0) {$X$};
    	   	\node (Z) at (1.5,0.7) {$Z$};
    	   	\node (W) at (1.5,-0.7) {$W$};
    	   	\node (Y) at (3.0,0) {$Y$};

    	    \path [->] (Z) edge (X);
    	    \path [->] (Z) edge (Y);
    	    \path [->] (X) edge (Y);
    	    \path [->] (X) edge (W);
    	    \path [->] (W) edge (Y);
    	    \path [<->] (Z) edge [dashed, bend left=40] (Y);
		\end{tikzpicture}
		\caption{}
		\label{fig:ex-ctf-measures:b}
	\end{subfigure}
	\hfill
	\begin{subfigure}[b]{0.24\columnwidth}
		\centering
		\begin{tikzpicture}[SCM,scale=0.9]
    	   	\node (X) at (0,0) {$X$};
    	   	\node (Z) at (1.5,0.7) {$Z$};
    	   	\node (W) at (1.5,-0.7) {$W$};
    	   	\node (Y) at (3.0,0) {$Y$};

    	    \path [->] (Z) edge (X);
    	    \path [->] (Z) edge (Y);
    	    \path [->] (X) edge (Y);
    	    \path [->] (X) edge (W);
    	    \path [->] (W) edge (Y);
    	    \path [<->] (X) edge [dashed, bend right=40] (W);
		\end{tikzpicture}
		\caption{}
		\label{fig:ex-ctf-measures:c}
	\end{subfigure}
	\hfill\null
	\caption{Three causal diagrams representing plausible structures in mediation analysis.}
	\label{fig:ex-ctf-measures}
\end{figure}
Armed with these results, we consider the identification problem in the next section.

\section{Counterfactual Identification from Observations and Experiments}\label{sec:identification}
In this section, we consider the identification of a counterfactual probability from a collection of observational and experimental distributions. This task can be seen as a generalization of that in \cite{lee:etal19} where the available data is the same, but the query is a causal effect $P_{\*x}(\*Y)$.
Let $\Z=\{\*Z_1, \*Z_2, \ldots \}, \*Z_j \subseteq \*V$, and assume that all of $\{P_{\*z_j}(\*V)\}_{\*z_j \in \Val{\*Z_j}, \*Z_j \in \Z}$ are available. Notice that $\*Z_j=\emptyset$ is a valid choice corresponding to $P(\*V)$ the observational (non-interventional) distribution. 

\begin{definition}[Counterfactual Identification]
A query $P(\*Y_*=\*y_*)$ is said to be identifiable from $\Z$ in $\G$, if $P(\*Y_*=\*y_*)$ is uniquely computable from the distributions $\{P_{\*z_j}(\*V)\}_{\*z_j \in \Val{\*Z_j}, \*Z_j \in \Z}$ in any causal model which induces $\G$.
\end{definition}

Given an arbitrary query $P(\*Y_*=\*y_*)$, we could express it in terms of \ctffactors{} by writing $P(\*Y_*=\*y_*)=\sum_{\*d_* \setminus \*y_*}P(\*D_* = \*d_*)$ where $\*D_*=\An{\*Y_*}$ and then using \cref{thm:ancestral-ctf-comp} to write $P(\*D_* = \*d_*)$ as a \ctffactor{}. For instance, the ancestral set $\*W_*=\{Y_x, X, Z, W\}$ with $\*w = \{y, x', z, w\}$ in \cref{eq:ett-ancestral-set} can be written in terms of \ctffactors{} as
\begin{align}
    P(Y_x=y, X_z=x', Z_w=z, W=w) = P(Y_x=y, X_z=x', W=w)P(Z_w=z).
\end{align}
The following lemma characterizes the relationship between the identifiability of $P(\*Y_*=\*y_*)$ and $P(\*D_* = \*d_*)$.

\begin{restatable}{lemma}{lemancnonid}\label{lem:factor-non-id-sum-non-id}
Let $P(\*W_* =\*w_*)$ be a \ctffactor{} and let $\*Y_* \subseteq \*W_*$ be such that $\*W_* = \An{\*Y_*}$. Then, $\sum_{\*w_* \setminus \*y_*}P(\*W_* =\*w_*)$ is identifiable from $\Z$ if and only if $P(\*W_* =\*w_*)$ is identifiable from $\Z$.
\end{restatable}

Once the query of interest is in \ctffactor{}-form, the identification task reduces to identifying smaller \ctffactors{} according to the c-components of $\G$. In this respect, \cref{thm:cc-comp-factorization} implies the following
\begin{corollary}\label{cor:non-id-comp-non-id-factor}
Let $P(\*W_* = \*w_*)$ be a \ctffactor{} and $\*C_j$ be a c-component of $\G[\*V(\*W_*)]$. Then, if $P(\*C_{j*}=\*c_{j*})$ is not identifiable, $P(\*W_* = \*w_{*})$ is also not identifiable.
\end{corollary}
\begin{proof}
Assume for the sake of contradiction that $P(\*C_{j*}=\*c_{j*})$ is not identifiable but $P(\*W_* = \*w_{*})$ is. Then, by \cref{thm:cc-comp-factorization}, the former is identifiable from the latter, a contradiction.
\end{proof}

Let us consider the causal diagrams in \cref{fig:ex-ctf-measures} and the counterfactual $Y_{x_1,W_{x_0}}=y, X=x$, with $x_0, x_1, x \in \Val{X}$, used to define quantities for fairness analysis in \cite{zhang:bar18a} (e.g., $Y_{x_1,W_{x_0}}=y | X=x$): 
\begin{align}
    &\hspace{-2em}P(Y_{x_1,W_{x_0}}=y, X=x) \nonumber\\
    &=\sum\nolimits_{w}P(Y_{x_1,w}=y, W_{x_0}=w, X=x) &\text{Unnesting}\\
    &=\sum\nolimits_{w, z}P(Y_{x_1,w}=y, W_{x_0}=w, X=x, Z=z) & \text{Complete ancestral set}\\
    &=\sum\nolimits_{w, z}P(Y_{x_1,w,z}=y, W_{x_0}=w, X_{z}=x, Z=z) & \text{Write in \ctffactor{}-form}
\end{align}
Due to the particular c-component structure of each model, we can factorize $P(Y_{x_1,w,z}=y, W_{x_0}=w, X_{z}=x, Z=z)$ according to each model as:
\begin{align}
    &P(Y_{x_1,w,z}=y)P(W_{x_0}=w)P(X_{z}=x)P(Z=z), \label{eq:factors:1}\\
    &P(Y_{x_1,w,z}=y, Z=z)P(W_{x_0}=w)P(X_{z}=x), \text{ and}\label{eq:factors:2}\\
    &P(Y_{x_1,w,z}=y)P(W_{x_0}=w, X_{z}=x)P(Z=z).\label{eq:factors:3}
\end{align}
The question then becomes, whether \ctffactors{} corresponding to individual c-components can be identified from the available input. In this example, all factors in \cref{eq:factors:1} and \cref{eq:factors:2} are identifiable from $P(\*V)$. For \cref{eq:factors:2} in particular, they are given by
\begin{align}
    P(Y=y, Z=z \mid W=w,X=x_1)P(W=w \mid X=x_0)P(X=x \mid Z=z).
\end{align}
In contrast, the factor $P(W_{x_0}{=}w, X_{z}{=}x)$ in \cref{eq:factors:3} (model \cref{fig:ex-ctf-measures:c}) is only identifiable if $x{=}x_0$.
The following definition and theorem characterize the factors that can be identified from $\Z$ and $\G$.

\begin{figure}
    \centering
    \begin{subfigure}[b]{0.29\columnwidth}
		\centering
		\begin{tikzpicture}[SCM,scale=0.9]
    	   	\node (Z) at (1.25,0.7) {$Z$};
    	   	\node (X) at (0,0) {$X$};
    	   	\node (Y) at (2.5,0) {$Y$};

    	    \path [->] (X) edge (Z);
    	    \path [->] (Z) edge (Y);
    	    \path [->] (X) edge (Y);
    	    \path [<->] (X) edge [dashed, bend left=30] (Z);
		\end{tikzpicture}
		\caption{The factor $P(Z_x=z)$ is identifiable only if $\{X\} \in \Z$.}
		\label{fig:nde-confounded}
	\end{subfigure}
	\hfill
	\begin{subfigure}[b]{0.31\columnwidth}
		\centering
		\begin{tikzpicture}[SCM,scale=0.9]
    	   	\node (X) at (0,0) {$X$};
    	   	\node (Y) at (2.5,0) {$Y$};

    	    \path [->] (X) edge (Y);
    	    \path [<->] (X) edge [dashed, bend left=40] (Y);
		\end{tikzpicture}
		\caption{The factor $P(Y_x=y, X=x')$ is inconsistent.}
		\label{fig:bow}
	\end{subfigure}
	\hfill
	\begin{subfigure}[b]{0.33\columnwidth}
		\centering
		\begin{tikzpicture}[SCM,scale=0.9]
    	   	\node (X1) at (0,0) {$X$};
    	   	\node (W) at (1.2,0) {$W$};
    	   	\node (X2) at (2.4,0) {$Z$};
    	   	\node (Y) at (3.6,0) {$Y$};
    	
    	    \path [->] (X1) edge (W);
    	    \path [->] (W) edge (X2);
    	    \path [->] (X2) edge (Y);
    	    \path [->] (W) edge[bend right=30] (Y);
		\end{tikzpicture}
		\caption{The factor $P(W_x=w,W_{x'}=w')$ is inconsistent.}
		\label{fig:markov-non-id-ett}
	\end{subfigure}
	\caption{Examples of causal diagrams and inconsistent \ctffactors{} derived from them.}
\end{figure}
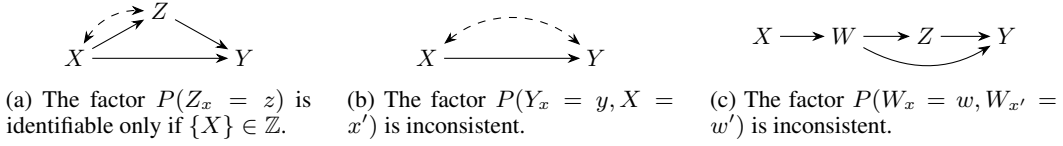

\begin{definition}[Inconsistent \ctffactor{}]\label{def:inconsistent-factor}
$P(\*W_* =\*w_*)$ is an inconsistent \ctffactor{} if it is a \ctffactor{}, $\G[\*V(\*W_*)]$ has a single c-component, and one of the following situations hold:
\begin{enumerate}[label={(\roman*)},nolistsep,nosep,leftmargin=2em,itemsep=0.2em]
    \item there exist $W_\*t \in \*W_*, Z \in \*T \cap \*V(\*W_*)$ such that $z \in \*t, z' \in \*w_*$ and $z \neq z'$, or 
    \item there exists $W_{i[\*t_i]}, W_{j[\*t_j]} \in \*W_*$ 
    and $T \in \*T_i \cap \*T_j$ such that $t \in \*t_1, t' \in \*t_2$ and $t \neq t'$.
\end{enumerate}
\end{definition}

\begin{restatable}[\Ctffactor{} identifiability]{theorem}{thminconsistentid}\label{thm:consistent-zid}
A \ctffactor{} $P(\*W_* =\*w)$ is identifiable from $\Z$ if and only if it is consistent. If consistent, let $\*W=\*V(\*W_{*})$ and $\*W' = \*V \setminus \*W$; then $P(\*W_* = \*w_*)$ is equal to $P_{\*w'}(\*w)$ where $\*w$ and $\*w'$ are consistent with $\*w_{*} \cup \bigcup_{\{W_{\pai{w}} \in \*W_{*}\}}\pai{w}$.
\end{restatable}

Consider the $\textit{NDE}_{x \to x',Z}(Y)$ in \cref{fig:nde-confounded}, we can write
\begin{align}
    P(Y_{x'Z_x}=y)=\sum\nolimits_{z}P(Y_{x'z}=y, Z_x=z)=\sum\nolimits_{z}P(Y_{x'z}=y)P(Z_x=z).
\end{align}
While the factor $P(Y_{x'z}=y)$ is identifiable from $P(\*V)$ as $P(Y=y \mid X=x', Z=z)$, the second factor is identifiable only if experimental data on $X$ is available, as $P_x(z)$.

We can also verify that the factor $P(Y_x = y, X=x')$ in \cref{fig:bow} is inconsistent. For another example consider the ETT-like expression $P(Y_{x,z}=y, X=x', Z=z')$ in \cref{fig:markov-non-id-ett}, we have
\begin{align}
    P(Y_{xz}=y&, X=x', Z=z') \nonumber\\
    &=\sum\nolimits_{w,w'}P(Y_{xz}=y, X=x', Z=z', W_x=w, W=w')\\
    &=\sum\nolimits_{w,w'}P(Y_{xz}=y, X=x', Z_{w'}=z', W_x=w, W_{x'}=w')\\
    &=\sum\nolimits_{w,w'}P(Y_{xz}=y)P(X=x')P( Z_{w'}=z')P(W_x=w, W_{x'}=w'),
\end{align}
where the factor $P(W_x=w, W_{x'}=w')$ is inconsistent.

Using the results in this section, we propose the algorithm \textsc{ctfID} (\cref{alg:ctfID}) which given a set of counterfactual variables $\*Y_*$, corresponding values $\*y_*$, a collection of observational and experimental distributions $\Z$, and a causal diagram $\G$; outputs an expression for $P(\*Y_*=\*y_*)$ in terms of the specified distributions or \textsc{Fail} if the query is not identifiable from such input in $\G$. Line~\ref{alg:ctfid:simp} removes irrelevant subscripts from the query by virtue of \cref{lem:int-minimal}. Then, lines~\ref{alg:ctfid:check0} and \ref{alg:ctfid:check-red} look for inconsistent events and redundant events, respectively. Line~\ref{alg:ctfid:defW} finds the relevant \ctffactors{} consisting of a single c-component, as licensed by \cref{thm:ancestral-ctf-comp} and \cref{thm:cc-comp-factorization}. As long as the factors are consistent, and allowed by \cref{thm:consistent-zid}, lines~\ref{alg:ctfid:Bi}-\ref{alg:ctfid:endifidentify} carry out identification of the causal effect $P_{\*V \setminus \*C_i}(\*C_i)$ from the available distributions employing the algorithm \textsc{Identify} \cite{tian:pea02-general-id} as a subroutine.\footnote{For a running example of the \textsc{ctrID} and details on how to use \textsc{Identify}, see Appendix E.} 
The procedure fails if any of the factors $P(\*C_{i*}=\*c_{i*})$ is inconsistent or not identifiable from $\Z$. Otherwise, it returns the corresponding expression.

\begin{algorithm}[tb]
\caption{\textsc{ctfID}$(\*Y_*, \*y_*, \bm{\Z}, \G$)}
\label{alg:ctfID}

\textbf{Input}: $\G$ causal diagram over variables $\*V$; $\*Y_*$ a set of counterfactual variables in $\*V$; $\*y_*$ a set of values for $\*Y_*$; and available distribution specification $\Z$.

\textbf{Output}: $P(\*Y_* = \*y_*)$ in terms of available distributions or \textsc{Fail} if not identifiable from $\langle \G, \Z \rangle$.

\begin{algorithmic}[1]
\STATE let $\*Y_* \gets \im{\*Y_*}$.\label{alg:ctfid:simp}
\IFTHEN {there exists $Y_{\*x} \in \*Y_*$ with two or more different values in $\*y_*(Y_{\*x})$ or $Y_{y} \in \*Y_*$ with $\*y_*(Y_y) \neq y$ } { \textbf{return} 0. } \label{alg:ctfid:check0}
\IFTHEN {there exists $Y_{\*x} \in \*Y_*$ with two consistent values in $\*y_*(Y_{\*x})$ or $Y_{y} \in \*Y_*$ with $\*y_*(Y_y) = y$ } { remove repeated variables from $\*Y_*$ and values $\*y_*$. } \label{alg:ctfid:check-red}
\STATE let $\*W_* \gets \An{\*Y_*}$, 
    and let $\*C_{1*}, \ldots, \*C_{k*}$ be corresponding \ctffactors{} in $\G[\*V(\*W_*)]$. \label{alg:ctfid:defW}
\FOR{\textbf{each} $\*C_i$ s.t. $(\*C_{i*}=\*c_{i*})$ is not inconsistent, $\*Z \in \Z$ s.t. $\*C_i \cap \*Z = \emptyset$}\label{alg:ctfid:forloop1} \label{alg:ctfid:forQNX}
    \STATE let $\*B_i$ be the c-component of $\G_{\overline{\*Z}}$ such that $\*C_i \subseteq \*B_i$, compute $P_{\*V \setminus \*B_i}(\*B_i)$ from $P_{\*Z}(\*V)$.\label{alg:ctfid:Bi}
    \IF {$\textsc{Identify}(\*C_i, \*B_i, P_{\*V \setminus \*B_i}(\*B_i), \G)$ does not \textsc{Fail}}\label{alg:ctfid:callIdentify}
        \STATE let $P_{\*V \setminus \*C_i}(\*C_i) \gets \textsc{Identify}(\*C_i, \*B_i, P_{\*V \setminus \*B_i}(\*B_i), \G)$.
        \STATE let $P(\*C_{i*}=\*c_{i*}) \gets \left[P_{\*V \setminus \*C_i}(\*C_i)\right]_{(\*c_{i*} \cup \bigcup_{C_{\*t} \in \*C_{i*}}\pai{c})}$.
        \STATE move to the next $\*C_i$. \label{alg:ctfid:nextci}
    \ENDIF \label{alg:ctfid:endifidentify}
\ENDFOR
\IFTHEN {any $P(\*C_{i*}=\*c_{i*})$ is inconsistent or was not identified from $\Z$} {\textbf{return} \textsc{Fail}}.\label{alg:ctfid:fail}
\STATE \textbf{return} $P(\*Y_*=\*y_*) \gets \sum_{\*w_* \setminus \*y_*}\prod_i P(\*C_{i*}=\*c_{i*})$.\label{alg:ctfid:return}
\end{algorithmic}
\end{algorithm}

\begin{restatable}[\textsc{ctfID} completeness]{theorem}{ctfidcompleteness}\label{thm:ctfid-completeness}
A counterfactual probability $P(\*Y_*=\*y_*)$ is identifiable from $\Z$ and $\G$ if and only if \textsc{ctfID} returns an expression for it.
\end{restatable}

\section{Identification of Conditional Counterfactuals}\label{sec:conditional}
In this section we consider counterfactual quantities of the form $P(\*Y_*=\*y_* \mid \*X_*=\*x_*)$. It is immediate to write such a query as $P(\*W_*=\*w_*)$ with $\*W_*=\*Y_* \cup \*X_*$, and try to identify it using \textsc{ctfID}. Nevertheless, depending on the graphical structure, the original query may be identifiable even if the latter is not. To witness, consider the causal diagram in \cref{fig:cond-id-marginal-non-id} and the counterfactual $P(Y_x=y \mid Z_x=z, X=x')$, which can be written as $P(Y_x=y, Z_x=z, X=x')/\sum_{y}P(Y_x=y, Z_x=z, X=x')$.
Following the strategy explained so far, the numerator is equal to $P(Y_z=y)P(Z_x=z, X=x')$, where the second \ctffactor{} is inconsistent, and therefore not identifiable from $\Z$. Nevertheless, the conditional query is identifiable as
\begin{align}
    \frac{
        P(Y_{xz}=y)P(Z_x=z, X=x')
    }{
        P(Z_x=z, X=x')\sum_{y}P(Y_{xz}=y)
    }
    = P(Y_{xz}=y) = P(Y=y \mid Z=z, X=x).
\end{align}
\begin{figure}
    \centering
    \begin{subfigure}[b]{0.48\columnwidth}
		\centering
		\begin{tikzpicture}[SCM,scale=0.9]
    	   	\node (X) at (0,0) {$X$};
    	   	\node (Z) at (1.85,0) {$Z$};
    	   	\node (Y) at (3.5,0) {$Y$};

    	    \path [->] (X) edge (Z);
    	    \path [->] (Z) edge (Y);
    	    \path [->] (X) edge [bend right=20] (Y);
    	    \path [<->] (X) edge [dashed, bend left=40] (Z);
    	    
		\end{tikzpicture}
		\caption{While $P(Y_x{=}y \mid Z_x{=}z, X{=}x')$ is identifiable from $\Z$, $P(Y_x{=}y, Z_x{=}z, X{=}x')$ is not.}
		\label{fig:cond-id-marginal-non-id}
	\end{subfigure}
	\hfill
	\begin{subfigure}[b]{0.48\columnwidth}
		\centering
		\begin{tikzpicture}[SCM,scale=0.9]
    	   	\node (X) at (0,0) {$X$};
    	   	\node (Z) at (1.85,0) {$Z$};
    	   	\node (Y) at (3.5,0) {$Y$};

    	    \path [->] (X) edge (Z);
    	    \path [->] (Z) edge (Y);
    	    \path [->] (X) edge [bend right=20] (Y);
    	    \path [<->] (X) edge [dashed, bend left=20] (Y);
		\end{tikzpicture}
		\caption{$P(Y_x{=}y \mid Z_x{=}z, X{=}x')$ is not identifiable from $\Z$ because of the factor $P(Y_{xz}{=}y, X{=}x')$.}
		\label{fig:cond-non-id}
	\end{subfigure}
    \caption{Examples of conditional queries.}
\end{figure}
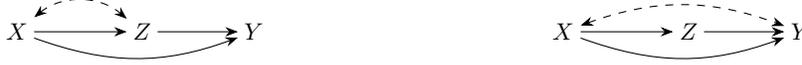
To characterize such simplifications of the query, we look at the causal diagram paying special attention to variables after the conditioning bar that are also ancestors of those before. Let $\*X_*(W_{\*t})=\*V(\im{\*X_*} \cap \An{W_{\*t}})$, that is, the primitive variables in $\*X_*$ that are also ancestors of $\*W_{\*t}$.
\begin{definition}[Ancestral components]\label{def:ancestral-components}
Let $\*W_*$ be a set of counterfactual variables, $\*X_* \subseteq \*W_*$, and $\G$ be a causal diagram. 
Then the ancestral components induced by $\*W_*$, given $\*X_*$, are sets $\*A_{1*}, \*A_{2*}, \ldots$ that form a partition over $\An{\*W_*}$, made of unions of ancestral the sets $\An[\G_{\underline{\*X_*(W_{\*t})}}]{W_{\*t}}, W_{\*t} \in \*W_*$. Sets $\An[\G_{\underline{\*X_*(W_{1[\*t_1]})}}]{W_{1[\*t_1]}}$ and $\An[\G_{\underline{\*X_*(W_{2[\*t_2]})}}]{W_{2[\*t_2]}}$ are put together if they are not disjoint or there exists a bidirected arrow in $\G$ connecting variables in those sets.
\end{definition}

\begin{restatable}{lemma}{thmqueryconditional}\label{lem:conditional}
Let $\*Y_*, \*X_*$ be two sets of counterfactual variables and let $\*D_*$ be the set of variables in the same ancestral component, given $\*X_*$, as any variable in $\*Y_*$, then
\begin{align}\label{eq:cond-ctf-factor}
    P(\*Y_*=\*y_* \mid \*X_*=\*x_*) = 
    \frac{\sum_{\*d_* \setminus (\*y_* \cup \*x_*)}P(\bigwedge_{D_\*t \in \*D_*}\*D_{\pai{d}} = d)}
    {\sum_{\*d_* \setminus \*x_*}P(\bigwedge_{D_\*t \in \*D_*}\*D_{\pai{d}} = d)},
\end{align}
where $\pai{d}$ is consistent with $\*t$ and $\*d_*$, for each $D_{\*t} \in \*D_{*}$.
Moreover, $P(\*Y_*=\*y_* \mid \*X_*=\*x_*)$ is identifiable from $\Z$ if and only if $P(\bigwedge_{D_\*t \in \*D_*}\*D_{\pai{d}} = d)$ is identifiable from $\Z$.
\end{restatable}

\begin{algorithm}[tb]
\caption{\textsc{cond-ctfID}$(\*Y_*, \*y_*, \*X_*, \*x_*, \bm{\Z}, \G$)}
\label{alg:condctfID}

\textbf{Input}: $\G$ causal diagram over variables $\*V$; $\*Y_*, \*X_*$ a set of counterfactual variables in $\*V$; $\*y_*, \*x_*$ a set of values for $\*Y_*$ and $\*X_*$; and available distribution specification $\Z$.

\textbf{Output}: $P(\*Y_* {=} \*y_* \mid \*X_* {=} \*x_*)$ in terms of available distributions or \textsc{Fail} if non-ID from $\langle \G, \Z \rangle$.

\begin{algorithmic}[1]
\STATE Let $\*A_{1*}, \*A_{2*}, \ldots$ be the ancestral components of $\*Y_* \cup \*X_*$ given $\*X_*$.\label{alg:condctfid:anccomp}
\STATE Let $\*D_*$ be the union of the ancestral components containing a variable in $\*Y_*$ and $\*d_*$ the corresponding set of values.\label{alg:condctfid:union}
\STATE let $Q \gets \textsc{ctfID}(\bigcup_{D_\*t \in \*D_*}\*D_{\pai{d}}, \*d_*, \Z, \G)$.\label{alg:condctfid:callctfid}
\RETURN $\sum_{\*d_* \setminus (\*y_* \cup \*x_*)}Q/\sum_{\*d_* \setminus \*x_*}Q$.
\end{algorithmic}
\end{algorithm}

Using the notion of ancestral components and \cref{lem:conditional}, we propose a conditional version of \textsc{ctfID} (\cref{alg:condctfID}). 
Due to \cref{lem:conditional}, it is easy to see that \textsc{cond-ctfID} is complete.

\begin{restatable}[\textsc{cond-ctfID} completeness]{theorem}{condctfidcompleteness}\label{thm:condctfid-completeness}
A counterfactual probability $P(\*Y_*=\*y_* \mid \*X_*=\*x_*)$ is identifiable from $\Z$ and $\G$ if and only if \textsc{cond-ctfID} returns an expression for it.
\end{restatable}

\section{Relation with ID* \cite{shpitser:pea07} and gID \cite{lee:etal19}}

In \cref{alg:ids}, we rewrite ID* \cite{shpitser:pea07} with the notation used in this paper and modify it for this task. Specifically, instead of returning $P_{\*x}(\bigwedge_{S_{\*t} \in \*S}S)$ in the last line, we invoke \textsc{gID} \cite{lee:etal19} to try to identify this effect from the available input distributions $\Z$. 
\begin{algorithm}[tb]
\caption{\textsc{gID*}$(\*Y_*, \*y_*, \bm{\Z}, \G)$}
\label{alg:ids}

\textbf{Input}: $\G$ a causal diagram, $(\*Y_*=\*y_*)$ a counterfactual event; and available distribution specification $\Z$.

\textbf{Output}: $P(\*Y_* = \*y_*)$ in terms of available distributions specified by $\Z$ or \textsc{Fail} if not identifiable from $\langle \G, \Z \rangle$.

\begin{algorithmic}[1]
\IFTHEN{$\*Y_* = \emptyset$}{\textbf{return} 1}
\IFTHEN{there exists $x \in \*y_*$ for some $X_{x'}$ with $x \neq x'$}{\textbf{return} 0}
\IFTHEN{there exists $x \in \*y_*$ for some $X_{x}$ with $x = x'$}{\textbf{return} \textsc{gID*}($\*Y_* \setminus \{X_x\}$, $\*y_* \setminus \{x\}, \Z, \G$)}
\STATE $(\G', (\*Y_*', \*y_*')) \gets \textbf{make-cg}(\G,(\*Y_*, \*y_*))$
\IFTHEN{$\*Y_*'$ is \textbf{INCONSISTENT}}{\textbf{return} 0}
\IF {$\G'$ has more than one c-component $\*S_{1}, \ldots, \*S_{k}$}
\RETURN $\sum_{\*v(\G) \setminus \*y_*'}\prod_{i}\textsc{gID*}((\*S_{i})_{\*v(\G) \setminus \*s_{i}}, \*y_*'(\*S_{i}), \Z, \G[\*S_{i}])$\label{alg:ids:subcall}
\ENDIF
\IF{$\G'$ has a single c-component $\*S$}
    \IFTHEN{$(\*S=\*s)$ is inconsistent}{\textbf{return} \textsc{Fail}} \label{alg:ids:fail1}
    \STATE let $\*x = \bigcup_{S_\*t \in \*S}\*t$
    \RETURN \textsc{gID}$(\*s, \*x, \Z, \G)$\label{alg:ids:gid}
\ENDIF
\end{algorithmic}
\end{algorithm}

In \cref{lem:ids-inconsistent-ctffactor} we show that from the failure of \textsc{gID*} we can establish the existence of a \ctffactor{} $(\*C_*=\*c_*)$ corresponding to a c-component of $\G[\*V(\An{\*Y_*})]$ that is either inconsistent or not identifiable from $\Z$ and $\G$. From this fact, the completeness of \textsc{gID*} follows as a corollary of the completeness of \textsc{ctfID}, that fails only under the same condition. Specifically, we can use \cref{thm:consistent-zid}, Lemma 5 (in Appendix C), and \cref{lem:factor-non-id-sum-non-id} as we did for \textsc{ctfID} to prove the same result for \textsc{gID*}. 

\begin{lemma}\label{lem:ids-inconsistent-ctffactor}
If \textsc{gID*} fails at line~\ref{alg:ids:fail1} for a counterfactual query $P(\*Y_*=\*y_*)$ given $\G$ and $\Z$, then there exists a c-component of $\G[\*V(\An{\*Y_*})]$ with \ctffactor{} ($\*C_*=\*c_*$) that is either inconsistent or not identifiable from $\Z$.
\end{lemma}

\begin{corollary}
A counterfactual probability $P(\*Y_*=\*y_*)$ is identifiable from $\Z$ and $\G$ if and only if \textsc{gID*} returns an expression for it.
\end{corollary}

\section{Conclusions}
We examined the identification of nested and non-nested counterfactuals from an arbitrary combination of observational and experimental distributions. We study nested counterfactuals within the SCM framework and prove several properties of counterfactual distributions (\cref{thm:ancestral-ctf-comp}, \ref{thm:cc-comp-factorization}) together with the counterfactual unnesting theorem (\cref{thm:counterfactualunnesting}). Moreover, we developed a graphical condition (\cref{def:inconsistent-factor}, \cref{thm:consistent-zid}) and an efficient algorithm for identifying marginal counterfactuals (\cref{alg:ctfID}) and proved their sufficiency and necessity (\cref{thm:ctfid-completeness}). Lastly, we reduce the identification of conditional counterfactuals to that of marginal ones (\cref{lem:conditional}) and give a corresponding complete algorithm (\cref{alg:condctfID}, \cref{thm:condctfid-completeness}) for this task. These results advance the state of the art by allowing for nested counterfactuals and relaxing the requirements on the data available to the analyst.

\bibliographystyle{apalike}
\bibliography{main.bib}

\end{document}